\documentclass[10pt,twocolumn,letterpaper]{article}
\usepackage[accsupp]{axessibility}  
\usepackage{cvpr}
\usepackage{graphicx}
\usepackage[english]{babel}
\usepackage[utf8]{inputenc}
\usepackage{amsmath}
\usepackage{amsthm}
\usepackage{mathtools}
\usepackage{caption}
\usepackage{subcaption}
\usepackage{amssymb}
\usepackage{bbm}
\usepackage{makecell}
\usepackage{amsfonts}
\usepackage{multirow}
\usepackage{array, booktabs, caption}
\usepackage{enumerate}
\usepackage{capt-of}
\usepackage[ruled,vlined]{algorithm2e}
\usepackage{float}
\usepackage{array}
\usepackage{listings}
\usepackage[svgnames]{xcolor}
\usepackage{cite}
\usepackage{pifont}
\usepackage{bm}
\usepackage{enumitem}
\usepackage[toc]{appendix}
\usepackage{fancybox,framed}
\usepackage[italicdiff]{physics}
\usepackage[pagebackref,breaklinks,colorlinks]{hyperref}
\hypersetup{
	colorlinks = true,
    linkcolor = blue
}
\usepackage[capitalize]{cleveref}
\crefname{section}{Sec.}{Secs.}
\Crefname{section}{Section}{Sections}
\Crefname{table}{Table}{Tables}
\crefname{table}{Tab.}{Tabs.}

\newcommand{\diff}{\mathop{}\!\mathrm{d}}

\newcommand{\expt}{\mathbb{E}}
\newcommand{\mcal}[1]{\mathcal{#1}}

\newcommand{\cmark}{\ding{51}}%
\newcommand{\xmark}{\ding{55}}%

\DeclareMathOperator*{\argmax}{arg\,max}

\newtheorem{theorem}{Theorem}
\newtheorem*{theorem*}{Theorem}
\newtheorem*{lemma*}{Lemma}

\newtheorem{definition}{Definition}

\newtheorem{lemma}{Lemma}

\theoremstyle{remark}

\theoremstyle{definition}

\def\cvprPaperID{9150} 
\def\confName{CVPR}
\def\confYear{2022}

\begin{document}
\title{On Learning Contrastive Representations for Learning with Noisy Labels}

\author{
Li Yi\textsuperscript{1}\quad
Sheng Liu\textsuperscript{2}\quad
Qi She\textsuperscript{3}\quad
A. Ian McLeod\textsuperscript{1}\quad
Boyu Wang$^*$\textsuperscript{1, 4}\\
\textsuperscript{1}University of Western Ontario, 
\textsuperscript{2}NYU Center for Data Science\\
\textsuperscript{3}ByteDance Inc. \quad
\textsuperscript{4}Vector Institute\\
{\tt\small lyi7@uwo.ca}
\quad
{\tt\small shengliu@nyu.edu}
\quad
{\tt\small sheqi1991@gmail.com}
\\
{\tt\small aimcleod@uwo.ca}
\quad
{\tt\small bwang@csd.uwo.ca}
}

\maketitle

\footnotetext[1]{Corresponding author}

\begin{abstract}
    Deep neural networks are able to memorize noisy labels easily with a softmax cross entropy (CE) loss.
    Previous studies attempted to address this issue focus on incorporating a  noise-robust loss function to the CE loss.
    However, the memorization issue is alleviated but still remains due to the non-robust CE loss.
    To address this issue, we focus on learning robust contrastive representations of data on which the classifier is hard to memorize the label noise under the CE loss.
    We propose a novel contrastive regularization function to learn such representations over noisy data where label noise does not dominate the representation learning.
    By theoretically investigating the representations induced by the proposed regularization function, we reveal that the learned representations keep information related to true labels and discard information related to corrupted labels. 
    Moreover, our theoretical results also indicate that the learned representations are robust to the label noise.
    The effectiveness of this method is demonstrated with experiments on benchmark datasets.
    Code is available at: \url{https://github.com/liyi01827/noisy-contrastive}
\end{abstract}

\begin{figure}[t]
    \centering
    \includegraphics[height=4.69cm]{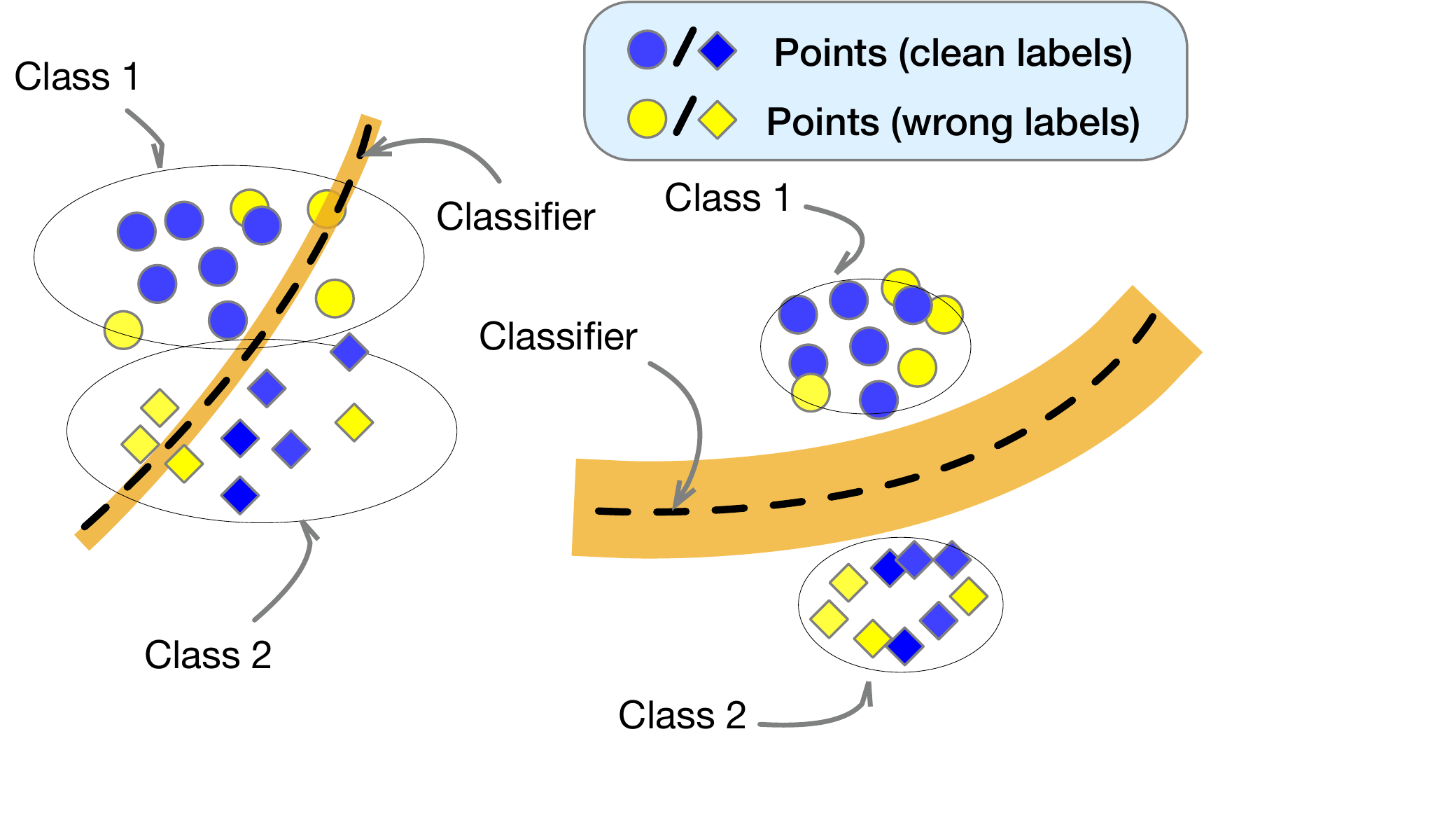}
    \caption{Illustration of the proposed method with noisy labels.
    Black curves are the best classifiers that are learned during training.
    \textbf{Left:} Deep networks without contrastive regularization.
    \textbf{Right:} Deep networks with contrastive regularization.  
    Two classes are better separated by deep networks that
    points with the same class are pulled into a tight cluster and 
    clusters are pushed away from each other.}
    \label{fig:ctrr}
\end{figure}

\section{Introduction}
The successes of deep neural networks \cite{DBLP:conf/cvpr/resnethe,fastrcnn} largely rely on availability of correctly labeled large-scale datasets that are prohibitively expensive and time-consuming to collect \cite{DBLP:conf/cvpr/XiaoXYHW15}.
Approaches to addressing this issue includes: acquiring labels from crowdsourcing-like platforms or non-expert labelers or  other unreliable sources \cite{ldmi, DBLP:conf/nips/ZhangS18} 
but while theses methods can reduce the labeling cost, label noise is inevitable.
Due to the over-parameterization of deep networks \cite{DBLP:conf/cvpr/resnethe}, examples with noisy labels can ultimately be memorized with a cross entropy loss\cite{DBLP:conf/nips/LiuNRF20, DBLP:conf/icml/ArpitJBKBKMFCBL17, DBLP:conf/nips/coresets}, which is known as the \emph{memorization effect} \cite{DBLP:conf/iclr/ZhangBHRV17, DBLP:conf/nips/MaennelATBBGK20}, leading to poor performance \cite{DBLP:conf/iclr/ZhangBHRV17}.
Therefore, it is important to develop methods that are robust to the label noise.

Cross entropy (CE) loss is widely used as a loss function for image classification tasks due to its strong performance on clean training data~\cite{DBLP:journals/corr/abs-2007-08199}
but it is not robust to label noise.
When labels in training data are corrupted, the performance drops 
\cite{DBLP:conf/nips/Ben-DavidBCP06, DBLP:journals/ml/Ben-DavidBCKPV10}.
Given the memorization effect of deep networks, training on noisy data with the CE loss results in the representations of the data clustered in terms of their noisy labels instead of the ground truth.
Thus, the final layer of the deep networks cannot find a good decision boundary from these noisy representations.

To overcome the memorization effect, noise-robust loss functions have been actively studied in the literature~\cite{DBLP:journals/tcyb/ManwaniS13, DBLP:conf/nips/ZhangS18, DBLP:conf/iccv/0001MCLY019, DBLP:conf/ijcai/FengSLL0020}.
They aim to design noise-robust loss functions in a way such that they achieve small loss on clean data and large loss on wrongly labeled data.
However, it has been empirically shown that being robust alone is not sufficient for a good performance as it also suffers from the \emph{underfitting} problem \cite{DBLP:conf/icml/MaH00E020}.
To address this issue, these noise-robust loss functions have to be explicitly or implicitly jointly used with the CE loss, which brings a trade-off between non-robust loss and robust loss.
As a result, the memorization effect is alleviated but still remains due to the non-robust CE loss.

In this paper, we tackle this problem from a different perspective. Specifically, we investigate contrastive learning and the effect of the clustering structure for learning with noisy labels.
Owing to the power of contrastive representation learning methods \cite{simsiam, byol, DBLP:conf/nips/ChenKSNH20, DBLP:conf/nips/KhoslaTWSTIMLK20, mocov2}, learning contrastive representations has been extensively applied on various tasks \cite{ctr_objectdetection, ctr_rl, ctr_videoaudio}.
The key component of contrastive learning is positive contrastive pair $(x_1,x_2)$.
Training a contrastive objective encourages the representations of $x_1,x_2$ to be closer.
In supervised classification tasks, correct positive contrastive pairs are formed by examples from the same class.
When label noise exists, defining contrastive pairs in terms of their noisy labels results in adverse effects.
Encouraging representations from different classes to be closer makes it even more difficult to separate  images of different classes.
Similar to our attempt to learn contrastive representations from noisy data, previous work has focused on reducing the adverse effects by re-defining contrastive pairs according to their pseudo labels \cite{Li_2021_ICCV, DBLP:journals/data/CiortanDP21, DBLP:conf/cvpr/0001L21, DBLP:conf/iclr/0001XH21}.
However, pseudo labels can be unreliable, and then wrong contrastive pairs are inevitable and can dominate the representation learning.

To address this issue, we propose a new contrastive regularization function that does not suffer from the adverse effects.
We theoretically investigate benefits of representations induced by the proposed contrastive regularization function from two aspects.
First, the representations of images  keep information related to true labels and discard information related to corrupted labels.
Second, we  theoretically  show that the classifier is hard to memorize corrupted labels given the learned representations, which demonstrates that our representations are robust to label noise.
Intuitively, learning such contrastive representations of data helps combat the label noise.
If data points are clustered tightly in terms of their true labels, then it makes the classifier hard to draw a decision boundary to separate the data in terms of their corrupted labels.
We illustrate this intuition in Figure \ref{fig:ctrr}.
Our main contributions are as follows.

\begin{itemize}
\setlength\itemsep{0em}
    \item We theoretically analyze the representations induced by the contrastive regularization function, showing that the representations keep information related to true labels and discard information related to corrupted labels. Moreover, we formally show that representations with insufficient corrupted label-related information are robust to label noise.
    \item We propose a novel algorithm over  data with noisy labels to learn contrastive representations, and provide gradient analysis to show that correct contrastive pairs can dominate the representation learning.
    \item We empirically show that our method can be applied with existing label correction techniques and noise-robust loss functions to further boost the performance.
    We conduct extensive experiments to demonstrate the efficacy of our method.
\end{itemize}

\section{Theoretical Analysis}  \label{sec:method}
In this section, we first introduce some notations and we then investigate the benefits of  representations learned by the contrastive regularization function.
\subsection{Preliminaries}
We use uppercases $X,Y,\dotsc$ to represent random variables, calligraphic letters $\mcal{X},\mathcal{Y},\dotsc$ to represent sample spaces, and lowercases $x,y,\dotsc$ to represent their realizations.
Let $X$ be input random variable and $Y$ be its true label.
We use $\tilde Y$ to denote the wrongly-labeled random variable that is not equal to $Y$. 
The entropy of the random variable $Y$ is denoted by $H(Y)$ and the mutual information
of $X$ and $Y$ is $I(X,Y)$.

Contrastive learning aims to learn representations of data that only the data from the same class have similar representations.
In this paper, we propose to learn the representations by introducing the following contrastive regularization function over all examples $\{(x_i,y_i)\}$ from $\mcal{X}\times \mcal{Y}$ and $y_i$ is the ground truth.
\begin{equation} \label{loss:ctr}
    \mcal{L}_\text{ctr}(x_i,x_j)=-\big(\langle\,\tilde q_i, \tilde z_j \rangle + \langle\,\tilde q_j, \tilde z_i \rangle \big)\mathbbm{1}\{y_i=y_j\},
\end{equation}
where $\tilde q_k=\frac{q_k}{\norm{q_k}_2}$ and $\tilde z_k=\frac{z_k}{\norm{z_k}_2}$.
Following SimSiam \cite{simsiam}, we define $q=h(f(x))$, $z=\texttt{stopgrad}(f(x))$, $f$ is an encoder network consisting of a backbone network and a projection MLP, and $h$ is a prediction MLP.
Minimizing Eq.~(\ref{loss:ctr}) on $\{(x_i,y_i),(x_j,y_j)\}$ pulls representations of $x_i$ and $x_j$ closer if $y_i=y_j$.
The designs of the stop-gradient operation and $h$ applied on representations are mainly to avoid trivial constant solutions.

\subsection{The Benefits of Representations Induced by Contrastive Regularization} \label{sec:thm}

We first relate the solutions that minimize Eq.~(\ref{loss:ctr}) to a mutual information $I(Z;X^+)=\iint p(z,x^+)\log{\frac{p(z|x^+)}{p(z)}} \diff x^+ \diff z$, where $z=f(x)$ and $x^+$ is from the same class as $x$.

\begin{theorem} \label{thm:relationship}
Representations $Z$ learned by minimizing Eq.~(\ref{loss:ctr}) maximizes the mutual information $I(Z;X^+)$.
\end{theorem}

Theorem~\ref{thm:relationship} reveals the equivalence between the contrastive learning and mutual information maximization. Intuitively,  Eq.~(\ref{loss:ctr}) encourages to pull representations from the same class together and push those from different classes apart.
The estimate of $z$ conditioned on $x^+$ is more accurate than random guessing because the representation $z$ of $x$ is similar to the representation of $x^+$.
Thus the pointwise mutual information $\log{\frac{p(z|x^+)}{p(z)}}$ increases by minimizing Eq.~(\ref{loss:ctr}).

We denote $Z^\star = \argmax_{Z_\theta} I(Z_\theta, X^+)$ by the representation that maximizes the mutual information, where $Z_\theta$ is a representation of $X$ parameterized by the neural network $f$ with parameters $\theta$. 
To understand what $Z^\star$ is learned from inputs and to show that $Z^\star$ is noise-robust, we introduce the notion of \emph{$(\epsilon, \gamma)$-distribution}:

\begin{definition}[$(\epsilon, \gamma)$-distribution] \label{assump:a1}
A distribution $D(X,Y,\tilde Y)$ is called $(\epsilon, \gamma)$-Distribution if there exists $\gamma \gg \epsilon >0$ such that
\begin{equation} \label{eq:asump11}
    I(X;Y|X^+) \leq \epsilon,
\end{equation}
and
\begin{equation} \label{eq:asump2}
    I(X;\tilde Y|X^+) > \gamma.
\end{equation}
\end{definition}

Eq.~(\ref{eq:asump11}) characterizes the connection between images and their true labels.
If we already know an image $X^+$, then there is the limited extra information related to the true label by additionally knowing $X$.
We use a small number $\epsilon$ to restrict this additional information gain.
Eq.~(\ref{eq:asump2}) characterizes the connection between those images and their corrupted labels.
By knowing an additional image $X^+$, the information $X$ contains about its corrupted label $\tilde Y$ is still larger than $\gamma$.
The above condition $\gamma \gg \epsilon >0$ states that images from the same class are much more similar with respect to the true label than the corrupted label.
As it is mentioned in \cite{DBLP:conf/colt/SridharanK08}, if there is a perfect prediction of $Y$ given $X^+$, then $\epsilon=0$.

\begin{figure} [t]
    \centering
    \includegraphics[width=0.47\textwidth]{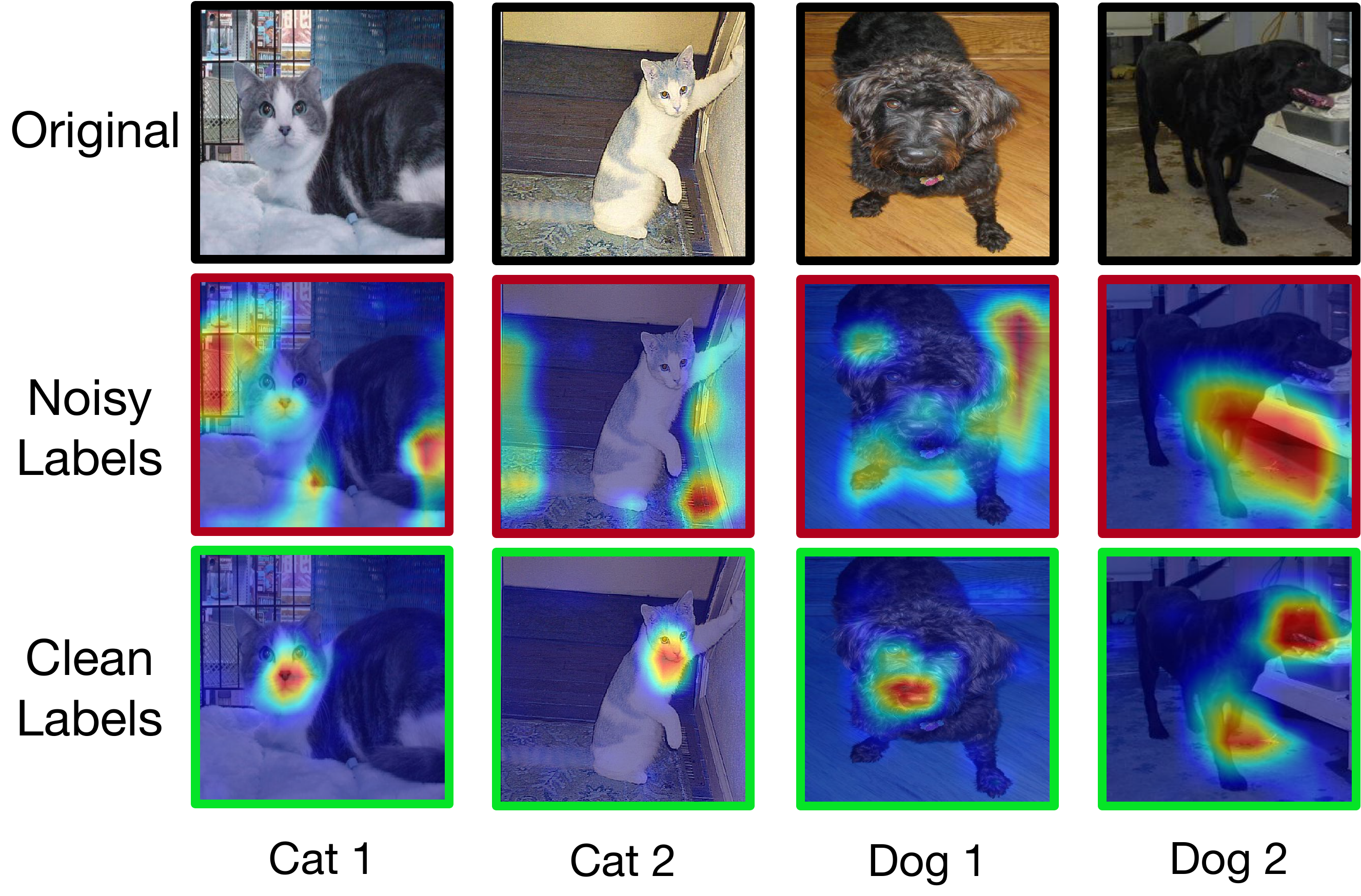}
    \caption{An example of Grad-CAM \cite{grad-cam} results of Resnet34 trained on noisy dataset with 40\% symmetric label noise and clean dataset, separately.
    When there is label noise, information related to corrupted labels captured by the model varies from image to image (e.g.\ window bars in Cat 1 v.s.\ floor and wall in Cat 2).
    When there is no label noise, information related to true labels are similar for images from the same class (e.g.\ cat face in Cat 1 v.s.\ cat face in Cat 2).}
    \label{fig:catdogs}
\end{figure}

We illustrate the intuitions behind Definition~\ref{assump:a1} in Figure~\ref{fig:catdogs}.
We use the Grad-CAM \cite{grad-cam} to highlight the important regions in the images for predictions.
The highlighted regions captured by the model are most related to labels.
For images with the same clean labels, their information related to true labels are similar.
For example, when Cat 1 and Cat 2 in Figure~\ref{fig:catdogs} are labeled as ``cat'', cat faces are captured as the true label-related information and they all look alike.
For images with corrupted labels, their information related to corrupted labels are quite different.
When Cat 1 and Cat 2 in Figure~\ref{fig:catdogs} are labeled as ``dog'',
the windows bars captured as the corrupted label-related information for Cat 1 is different from the floor and wall for Cat 2.

With the notion of $(\epsilon, \gamma)$-distribution, the following theorem help us understand the benefits of representations $Z^\star$ in depth.

\begin{theorem} \label{thm:1}
Given a distribution $D(X,Y,\tilde Y)$ that is $(\epsilon, \gamma)$-Distribution, we have 
\begin{align}
    &I(X;Y) - \epsilon  \leq I(Z^\star;Y) \leq I(X;Y), \label{eq2}\\
    &I(Z^\star;\tilde Y)  \leq I(X;\tilde Y) - \gamma + \epsilon. \label{eq3}
\end{align}
\end{theorem}
Given images $X$ and their labels $Y$, the mutual information $I(X;Y)$ is fixed.
The theorem states that the learned representations $Z^\star$ keep as much true label-related information as possible and  discard much corrupted label-related information.
Since the corrupted label-related information is discarded from the representations $Z^\star$, memorizing the corrupted labels based on $Z^\star$ is diminished.
Lemma~\ref{lem:1} establishes the lower bound on the expected error on \emph{wrongly-labeled} data.

\begin{figure*}[!h]
\includegraphics[width=\textwidth]{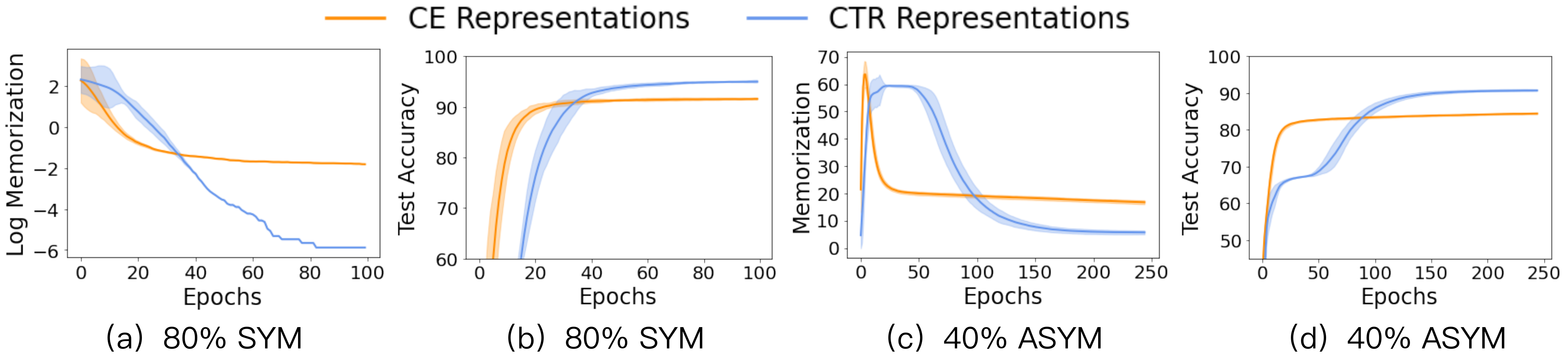}
\caption{Results of memorization of label noise and performance on test data on CIFAR-10 with 80\% symmetric label noise (SYM) and 40\% asymmetric label noise (ASYM).
The memorization is defined by the fraction of wrongly labeled examples whose predictions are equal to their labels.}
\label{fig:labelacc}
\end{figure*}

\begin{lemma} \label{lem:1}
Consider a pair of random variables $(X,\tilde Y)$.
Let $\hat Y$ be outputs of any classifier based on inputs $Z_\theta$,
and $\tilde e=\mathbbm{1}\{\hat Y \not= \tilde Y \}$, where $\mathbbm{1}\{A\}$ be the indicator function of event $A$. 
Then, we have
\begin{equation*}
    \expt[\tilde e] \geq \frac{H(\tilde Y)-I(Z_\theta;\tilde Y) -H(\tilde e)}{\log(|\mcal{\tilde Y}|) - 1}.
\end{equation*}
\end{lemma}
Lemma~\ref{lem:1} provides  a necessary condition on the success of learning with noisy labels based on representation learning and sheds new light on this problem by highlighting the role of minimizing $I(Z_\theta; \tilde{Y})$. To see this, note that small $I(Z_\theta; \tilde{Y})$ implies robustness to label noise since $\expt[\tilde e]$ is the expected error over the corrupted labels. On the other hand, when minimizing Eq.~(\ref{loss:ctr}), small $I(Z^\star; \tilde{Y})$ can be achieved as indicated by the upper bound in Eq.~(\ref{eq3}). In the meanwhile, the lower bound on $I(Z^\star;Y)$ in Eq.~(\ref{eq2}) also shows that $Z^\star$ can retain the discriminative information of the data to avoid a trivial solution to $I(Z_\theta; \tilde{Y})$ minimization (i.e., $Z_\theta$ is a constant representation).

While Lemma~\ref{lem:1} combined with Theorem~\ref{thm:1} indicates that $Z^\star$ is robust to label noise, the following Lemma  shows that $Z^\star$ can also avoid underfitting. Specifically, it implies that that a good classifier achieved under the clean distribution can also be achieved based on our representations $Z^\star$.

\begin{lemma} \label{lem:2}
Let $R(X)=\inf_{g} \expt_{X,Y}[\mcal{L}(g(X),Y)]$ be the minimum risk over the joint distribution $X\times Y$, where $\mcal{L}(p,y)=\sum_{i=1}^{\mcal{Y}}y^{(i)}\log p^{(i)}$ is a CE loss  and $g$ is a function mapping from input space to label space.
Let $R(Z^\star)=\inf_{g^\prime} \expt_{Z^\star,Y}[\mcal{L}(g^\prime(Z^\star),Y)]$ be the minimum risk over the joint distribution $Z^\star\times Y$ and $g^\prime$ maps from representation space to label space. Then,
    \begin{equation*}
        R(Z^\star) \leq R(X) + \epsilon.
    \end{equation*}
\end{lemma}

To show the robustness and performance of the contrastive (CTR) representation $Z^\star$, we empirically compare it  to the representation learned by the CE loss.
We first use clean labels to train neural networks with different loss functions.
Then we initialize the parameters of the final linear classifier and fine-tune the linear layer with noisy labels.
We denote the memorization by the fraction of corrupted examples whose predictions are equal to their labels.
Figure~\ref{fig:labelacc} illustrates the improved performance and robustness in terms of test accuracy and reduced memorization with the CTR representation.

Conventionally, the memorization of label noise increases as the training progresses \cite{DBLP:conf/nips/LiuNRF20,DBLP:conf/iclr/XiaL00WGC21}.
We remark that previous memorization is observed and proved in \emph{over-parameterized} models, where the ratio of the number parameters and the sample size is around $220$.
In their settings, the fraction of examples memorized by the model will increase.
However, the memorization in our setting is measured on a linear classifier on top of frozen data representations, where the ratio of the number parameters and the sample size is around $0.1$,  which is \emph{under-parameterized}.
This explains why Figure~\ref{fig:labelacc} shows that the memorization decreases as the training progresses.

\section{Algorithm} \label{sec:alg}
In practice, as we are only given a noisy data set, we do not know if a label is clean or not. Consequently,  simply minimizing Eq.~(\ref{loss:ctr}) can lead to deteriorated performances. To see this, note that Eq.~(\ref{loss:ctr}) is activated only when  $\mathbbm{1}\{y_i=y_j\}=1$. Thus, two representations from different classes will be pulled together when there are  noisy labels.

Since deep networks first fit examples with clean labels  and the probabilistic outputs of these examples are higher than examples with corrupted labels \cite{DBLP:conf/aistats/LiSO20, DBLP:conf/icml/ArpitJBKBKMFCBL17},
one straightforward approach to tackle this issue is to replace the indicator function with a more reliable criterion $\mathbbm{1}\{p_i^\top p_j \geq \tau\}$:
\begin{equation} \label{eq:ori_noisectr}
    \mcal{L}^\prime_\text{ctr}(x_i,x_j) = -\big(\langle\,\tilde q_i, \tilde z_j \rangle + \langle\,\tilde q_j, \tilde z_i \rangle \big)\mathbbm{1}\{p_i^\top p_j \geq \tau\},
\end{equation}
where $p_i$ is the probabilistic output produced by linear classifier on the representation of image $x_i$ and $\tau$ is a confidence threshold.
However, minimizing Eq.~(\ref{eq:ori_noisectr}) only helps representation learning during the early stage.
After that period, examples with corrupted labels will dominate the learning procedure since the magnitudes of gradient from correct contrastive pairs overwhelm that from wrong contrastive pairs.
In particular, given two clean examples $x_i,x_j$ with $y_i=y_j$ and a wrongly labeled example $x_m$ with $\tilde y_m=y_i=y_j$, during the early stage, representations $\tilde q_i^\top \tilde q_j \rightarrow 1$ and $\tilde q_i^\top \tilde q_m \approx 0$.
After the early stage, deep networks starts to fit wrongly labeled data.
At this moment, the wrong contrastive pairs $(x_i, x_m)$ and $(x_j, x_m)$ are wrongly pulled together and they impair  the representation learning instead of the correct pair $(x_i, x_j)$:
\begin{align} 
    &\norm{\frac{\partial \mcal{L}^\prime_\text{ctr}(x_i,x_m)}{\partial q_i}}_2^2
    =c_i(\underbrace{1-\tilde q_i^\top \tilde q_m }_{\approx 1}) \label{eq:gradcompare} \\
    &\gg c_i(\underbrace{1-\tilde q_i^\top \tilde q_j}_{\approx 0})=\norm{\frac{\partial \mcal{L}^\prime_\text{ctr}(x_i,x_j)}{\partial q_i}}_2^2, \nonumber
\end{align}
where $c_i=1/\norm{q_i}_2^2$ and we take $h$ as an identity function for simplicity. The proof is shown in supplementary materials.

To address this issue, we propose the following regularization function to avoid the negative effects from wrong contrastive pairs:
\begin{equation} \label{loss:noisyctr}
\begin{aligned}
&\mcal{\widetilde L}_\text{ctr}(x_i,x_j)= \\
&\bigg(\log{\big(1- \langle\,\tilde q_i, \tilde z_j \rangle \big)} +
     \log{\big(1- \langle\,\tilde q_j, \tilde z_i \rangle \big)} \bigg)\mathbbm{1}\{p_i^\top p_j \geq \tau\}
\end{aligned}
\end{equation}
Eq.~(\ref{loss:noisyctr}) still aims to learn similar representations for data with the same true labels.
Since the maximum  of Eq.~(\ref{loss:noisyctr}) is the same as that of Eq.~(\ref{loss:ctr}), our theoretical results about $Z^\star$ still hold.
Moreover, the gradient analysis of Eq.~(\ref{loss:noisyctr}) is given by
\begin{align} \label{eq:grad}
    \norm{\frac{\partial \mcal{\widetilde L}_\text{ctr}(x_i,x_j)}{\partial q_i}}_2^2 =c_i(1+\tilde q_i^\top \tilde q_j),
\end{align}
which indicates that the gradient in L2 norm increases if $\tilde q_i$ and $\tilde q_j$ approach to each other.
In other words, the gradient from the correct pair $(x_i, x_j)$ is larger than the gradient from the wrong pair $(x_i, x_m)$ ($1+\tilde q_i^\top \tilde q_j > 1+\tilde q_i^\top \tilde q_m \approx 1$) during the learning procedure.
Compared to the gradient given by Eq.~(\ref{eq:gradcompare}), our proposed regularization function does not suffer from the gradient domination by wrong pairs.
Meanwhile, the model does not overfit clean examples even though the gradients of Eq.~(\ref{loss:noisyctr}) from correct pairs are larger than wrong pairs.
As Eq.~(\ref{eq:gradcompare}) describes the gradient with respect to the representation, its magnitude can be viewed as the strength of pulling clean examples from the same class closer, which is \emph{not} directly related to overfitting to clean examples. Moreover, we use a separate linear layer on top of the representations as the classifier, thus as long as the gradients of the classification loss with respect to the parameters in the linear layer are not large on the clean examples, the model would not overfit to them.

Finally, the overall objective function is given by
\begin{equation} \label{eq:finalloss}
    \mcal{L}=\mcal{L}_\text{ce} + \lambda \mcal{\widetilde L}_\text{ctr}, 
\end{equation} 
where $\mcal{\widetilde L}_\text{ctr}$ serves as a contrastive regularization (CTRR) on representations and $\lambda$ controls the strength of the regularization.

\section{Experiments}
{\bf Datasets. }We evaluate our method on two artificially corrupted datasets CIFAR-10 \cite{cifar} and CIFAR-100 \cite{cifar}, and two real-world datasets ANIMAL-10N \cite{DBLP:conf/icml/SongK019} and Clothing1M \cite{clothing1m}.
CIFAR-10 and CIFAR1-00 contain $50,000$ training images and $10,000$ test images with $10$ and $100$ classes, respectively.
ANIMAL-10N has 10 animal classes and  $50,000$ training images with confusing appearances and $5000$ test images.
Its estimated noise level is around 8\%.
Clothing1M has a million training images and  $10,000$ test images with $14$ classes.
Its estimated noise level is around 40\%.

{\bf Noise Generation. }For CIFAR-10, we consider two different types of synthetic noise with various noise levels.
For symmetric noise, each label has the same probability of flipping to any other classes, and we randomly choose $r$ training data with their labels to be flipped for $r \in \{20\%, 40\%, 60\%, 80\%, 90\%\}$.
For asymmetric noise, following \cite{noiseagainstnoise}, we flip labels between TRUCK$\rightarrow$AUTOMOBILE, BIRD$\rightarrow$AIRPLANE, DEER$\rightarrow$HORSE, and CAT$\leftrightarrow$DOG.
we randomly choose $40\%$ training data with their labels to be flipped according to the asymmetric labeling rule.
For CIFAR-100, we also consider two different types of synthetic noise with various noise levels.
The generation for symmetric label noise is the same as that for CIFAR-10 with the noise level $r \in \{20\%, 40\%, 60\%, 80\%\}$.
To generate asymmetric label noise, we randomly sample $40\%$ data and flip their labels to the next classes.

\begin{table*}[ht]
    \centering
    \resizebox{\textwidth}{!}{
    \begin{tabular}{c|cccccc|c}
    \toprule
         \multirow{3}{*}{Method}  &  \multicolumn{7}{c}{CIFAR-10} \\ \cline{2-8}
         &\multicolumn{6}{c|}{Sym.} & \multicolumn{1}{c}{Asym.} \\
        & 0\% &20\% & 40\% & 60\% & 80\%  &90\%& 40\% \\
         \hline\hline
        CE & $93.97_{\pm 0.22}$ &$88.51_{\pm 0.17}$ &$82.73_{\pm 0.16}$ &$76.26_{\pm 0.29}$ &$59.25_{\pm 1.01}$ &$39.43_{\pm 1.17}$ &$83.23_{\pm 0.59}$  \\
        Forward & $93.47_{\pm 0.19}$ &$88.87_{\pm 0.21}$ &$83.28_{\pm 0.37}$ &$75.15_{\pm 0.73}$ &$58.58_{\pm 1.05}$ &$38.49_{\pm 1.02}$ &$82.93_{\pm 0.74}$  \\
        GCE &$92.38_{\pm 0.32}$ &$91.22_{\pm 0.25}$ &$89.26_{\pm 0.34}$ &$85.76_{\pm 0.58}$ &$70.57_{\pm 0.83}$ &$31.25_{\pm 1.04}$ &$82.23_{\pm 0.61}$  \\
        Co-teaching &  $93.37_{\pm 0.12}$ &$92.05_{\pm 0.15}$ &$87.73_{\pm 0.17}$ &$85.10_{\pm 0.49}$ &$44.16_{\pm 0.71}$ &$30.39_{\pm 1.08}$ &$77.78_{\pm 0.59}$  \\
        LIMIT & $93.47_{\pm 0.56}$ &$89.63_{\pm 0.42}$ &$85.39_{\pm 0.63}$ &$78.05_{\pm 0.85}$ &$58.71_{\pm 0.83}$ &$40.46_{\pm 0.97}$ &$83.56_{\pm 0.70}$  \\
        SLN &$93.21_{\pm 0.21}$ &$88.77_{\pm 0.23}$ &$87.03_{\pm 0.70}$ &$80.57_{\pm 0.50}$ &$63.99_{\pm 0.79}$ &$36.64_{\pm 1.77}$ &$81.02_{\pm 0.25}$  \\
        SL &$94.21_{\pm 0.13}$ &$92.45_{\pm 0.08}$ &$89.22_{\pm 0.08}$ &$84.63_{\pm 0.21}$ &$72.59_{\pm 0.23}$ &$51.13_{\pm 0.27}$ &$83.58_{\pm 0.60}$  \\
        APL & $93.97_{\pm 0.25}$ &$92.51_{\pm 0.39}$ &$89.34_{\pm 0.33}$ &$85.01_{\pm 0.17}$ &$70.52_{\pm 2.36}$ &$49.38_{\pm 2.86}$ &$84.06_{\pm 0.20}$  \\
        \midrule
        CTRR &$\mathbf{94.29_{\pm 0.21}}$&$\mathbf{93.05_{\pm 0.32}}$ &$\mathbf{92.16_{\pm 0.31}}$ &$\mathbf{87.34_{\pm 0.84}}$ &$\mathbf{83.66_{\pm 0.52}}$ &$\mathbf{81.65_{\pm 2.46}}$ &$\mathbf{89.00_{\pm 0.56}}$  \\
        \bottomrule
    \end{tabular}
    }
    \caption{Test accuracy on CIFAR-10 with different noise types and noise levels.
    All method use the same model PreAct ResNet18 and their best results are reported over three runs.}
    \label{tab:cifar}
\end{table*}

\begin{table*}[ht]
    \centering
    \resizebox{0.9\textwidth}{!}{
  \begin{tabular}{c|ccccc|c}
    \toprule
         \multirow{3}{*}{Method}  &  \multicolumn{5}{c}{CIFAR-100} \\ \cline{2-7}
         &\multicolumn{5}{c|}{Sym.} & \multicolumn{1}{c}{Asym.} \\
        & 0\% &20\% & 40\% & 60\% & 80\% & 40\%\\
         \hline\hline
        CE & $73.21_{\pm 0.14}$ &$60.57_{\pm 0.53}$ &$52.48_{\pm0.34}$ &$43.20_{\pm 0.21}$ &$22.96_{\pm 0.84}$&$44.45_{\pm 0.37}$    \\
        Forward & $73.01_{\pm 0.33}$ &$58.72_{\pm 0.54}$ &$50.10_{\pm 0.84}$ &$39.35_{\pm 0.82}$ &$17.15_{\pm 1.81}$& -   \\
        GCE &$72.27_{\pm 0.27}$ &$68.31_{\pm 0.34}$ &$62.25_{\pm 0.48}$ &$53.86_{\pm 0.95}$ &$19.31_{\pm 1.14}$ &$46.50_{\pm 0.71}$   \\
        Co-teaching &  $73.39_{\pm 0.27}$ &$65.71_{\pm 0.20}$ &$57.64_{\pm 0.71}$ &$31.59_{\pm 0.88}$ &$15.28_{\pm 1.94}$ & - \\
        LIMIT & $65.53_{\pm 0.91}$ &$58.02_{\pm 1.93}$ &$49.71_{\pm 1.81}$ &$37.05_{\pm 1.39}$ &$20.01_{\pm 0.11}$ & -   \\
        SLN &$63.13_{\pm 0.21}$ &$55.35_{\pm 1.26}$ &$51.39_{\pm 0.48}$ &$35.53_{\pm 0.58}$ &$11.96_{\pm 2.03}$ & -    \\
        SL &$72.44_{\pm 0.44}$ &$66.46_{\pm 0.26}$ &$61.44_{\pm 0.23}$ &$54.17_{\pm 1.32}$ &$34.22_{\pm 1.06}$ &$46.12_{\pm 0.47}$   \\
        APL & $73.88_{\pm 0.99}$ &$68.09_{\pm 0.15}$ &$63.46_{\pm 0.17}$ &$53.63_{\pm 0.45}$ &$20.00_{\pm 2.02}$&$52.80_{\pm 0.52}$  \\
        \midrule
        CTRR  &$\mathbf{74.36_{\pm 0.41}}$ &$\mathbf{70.09_{\pm 0.45}}$ &$\mathbf{65.32_{\pm 0.20}}$ &$\mathbf{54.20_{\pm 0.34}}$ &$\mathbf{43.69_{\pm 0.28}}$ &$\mathbf{54.47_{\pm 0.37}}$  \\ 
        \bottomrule
    \end{tabular}
    }
    \caption{Test accuracy on CIFAR-100 with different noise levels.
    All method use the same model PreAct ResNet18 and their best results are reported over three runs.}
    \vspace{-10 pt}
    \label{tab:cifar2}
\end{table*}

\begin{table}
    \centering
    \resizebox{0.42\textwidth}{!}{
    \begin{tabular}{c|c|c}
    \toprule
      Method & ANIMAL-10N& Clothing1M \\
      \hline
      CE & $83.18_{\pm 0.15}$ & $70.88_{\pm 0.45}$ \\
      Forward & $83.67_{\pm 0.31}$ & $71.23_{\pm 0.39}$\\
      GCE & $84.42_{\pm 0.39}$ & $71.34_{\pm 0.12}$\\
      Co-teaching & $85.73_{\pm 0.27}$ & $71.68_{\pm 0.21}$\\
      SLN & $83.17_{\pm 0.08}$ & $71.17_{\pm 0.12}$\\
      SL & $83.92_{\pm 0.28}$ & $72.03_{\pm 0.13}$\\
      APL & $84.25_{\pm 0.11}$ & $72.18_{\pm 0.21}$\\
      \midrule
      CTRR & $\mathbf{86.71_{\pm 0.15}}$ & $\mathbf{72.71_{\pm 0.19}}$\\
        \bottomrule
    \end{tabular}
    }
    \caption{Test accuracy on the real-world datasets ANIMAL-10N and Clothing1M.
    The results are obtained based on three different runs.}
        \vspace{-10 pt}
    \label{tab:animal}
\end{table}

{\bf Baseline methods. } To evaluate our method, we mainly compare our robust loss function to other robust loss function methods: 1) CE loss. 2) Forward correction \cite{DBLP:conf/cvpr/PatriniRMNQ17}, which corrects loss values by a estimated noise transition matrix.  3) GCE \cite{DBLP:conf/nips/ZhangS18}, which takes  advantages of both MAE loss and CE loss and designs a robust loss function.  4) Co-teaching \cite{DBLP:conf/nips/HanYYNXHTS18}, which maintains two networks and uses small-loss examples to update.  5) LIMIT \cite{DBLP:conf/icml/HarutyunyanRSG20}, which introduces noise to gradients to avoid memorization. 6) SLN \cite{noiseagainstnoise}, which adds Gaussian noise to noisy labels to combat label noise. 7) SL  \cite{DBLP:conf/iccv/0001MCLY019}, which uses CE loss and a reverse cross entropy loss (RCE) as a robust loss function. 8) APL (NCE+RCE) \cite{DBLP:conf/icml/MaH00E020}, which combines two mutually boosted robust loss functions for training.

{\bf Implementation details. }  We use a PreAct Resnet18 as the encoder for CIFAR datasets, and Resnet18 as the encoder for the two real-world datasets.
The project MLP and the prediction MLP are the same for all encoders.
Following SimSiam \cite{simsiam}, the projection MLP consists of $3$ layers which have $2048$  hidden dimensions and output $2048$-dimensional embeddings.
The prediction MLP consists of $2$ layers which have $512$ hidden dimensions and output $2048$-dimensional embeddings.
Following \cite{DBLP:conf/nips/ChenKSNH20}, we apply strong augmentations to learn data representations, where the strong augmentation includes Gaussian blur, color distortion, random flipping and random cropping.
We use weak augmentations to optimize the cross-entropy loss, which includes random flipping and random cropping.
More implementation details can be found in supplementary materials.

\begin{table*}[!ht]
    \centering
    \resizebox{\textwidth}{!}{
    \begin{tabular}{c|cccccc}
    \toprule
          \multirow{2}{*}{Regularization Functions}  & \multicolumn{6}{c}{CIFAR-10}  \\ 
          &0\% & 20\% & 40\% & 60\%  & 80\%& 90\% \\
         \hline\hline
        $\mcal{L}^\prime_\text{ctr}$(\ref{eq:ori_noisectr}) & $93.58_{\pm 0.11}$   & $86.05_{\pm 0.33}$   & $82.34_{\pm 0.25}$ &$74.35_{\pm 0.54}$ &$54.83_{\pm 1.00}$ &$40.96_{\pm 0.99}$ \\
        $\mcal{\widetilde L}_\text{ctr}$(\ref{loss:noisyctr}) &$\mathbf{94.29_{\pm 0.21}}$&$\mathbf{93.05_{\pm 0.32}}$ &$\mathbf{92.16_{\pm 0.31}}$ &$\mathbf{87.34_{\pm 0.84}}$ &$\mathbf{83.66_{\pm 0.52}}$ &$\mathbf{81.65_{\pm 2.46}}$  \\
        \bottomrule
    \end{tabular}
    }
    \caption{The performance of the model with respect to different regularization functions.}
    \label{abla:ctrloss} 
\end{table*}

\begin{table*}[!ht]
    \centering
    \resizebox{0.9\textwidth}{!}{
    \begin{tabular}{c|ccccc}
    \toprule
         \multirow{2}{*}{Contrastive Frameworks}  & \multicolumn{5}{c}{CIFAR-10}  \\ 
           & 20\% & 40\% & 60\%  & 80\%& 90\% \\
         \hline  \hline
        CTRR (SimSiam) &${93.05_{\pm 0.32}}$ &${92.16_{\pm 0.31}}$ &${87.34_{\pm 0.84}}$ &${83.66_{\pm 0.52}}$ &${81.65_{\pm 2.46}}$  \\
        CTRR (SimCLR) & $92.50_{\pm 0.35}$ &$90.12_{\pm 0.43}$ &$87.41_{\pm 0.83}$ &$84.96_{\pm 0.44}$ &$79.57_{\pm 1.32}$  \\
        CTRR (BYOL) &$93.31_{\pm 0.16}$ &$92.12_{\pm 0.16}$ &$88.71_{\pm 0.52}$ &$86.99_{\pm 0.59}$ &$84.31_{\pm 0.66}$  \\
        \bottomrule
    \end{tabular}
    }
    \caption{Extending our method to other contrasitve learning frameworks.}
    \vspace{-12 pt}
    \label{tab:ssl}
\end{table*}

\subsection{CIFAR Results}
Table \ref{tab:cifar} and Table \ref{tab:cifar2} show the results on CIFAR-10 and CIFAR-100 with various label noise settings.
We use PreAct Resnet18 \cite{DBLP:conf/cvpr/resnethe} for all methods and report the best test accuracy for them based on three runs.
Our method achieves the best performance on all tested noise settings.
The improvement is more substantial when the noise level is higher.
Especially when noise levels reach to $80\%$ or even $90\%$, our method significantly outperforms other methods.
For example, on CIFAR-10 with $r=90\%$, CTRR maintains a high accuracy of $81.65\%$, whereas the second best one is $49.65\%$. 

\subsection{ANIMAL-10N \& Clothing1M Results}
Table \ref{tab:animal} shows the results on the real-world datasets ANIMAL-10N and Clothing1M.
All methods use the same model and the best results are reported over three runs.
In order to be consistent with previous works for a fair comparison, we use a random initialized Resnet18 and an ImageNet pre-trained Resnet18 on ANIMAL-10N and Clothing1M, respectively, and the best results are reported over three runs. 
For Clothing1M, following \cite{DBLP:conf/iclr/LiSH20, noiseagainstnoise}, we randomly sample a balanced subset of $20.48$K images from the noisy training data and report performance on $10$K test images.
Our method is superior to other baselines on the two real-world datasets.

\begin{figure}[!t] 
\begin{subfigure}{.25\textwidth}
  \includegraphics[width=\textwidth]{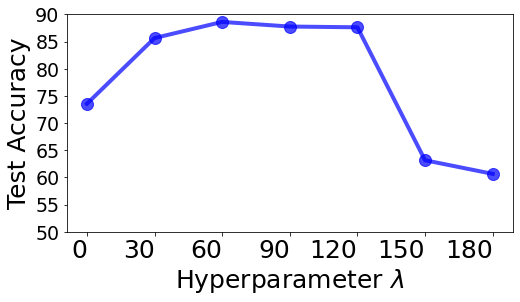}
\end{subfigure}%
\begin{subfigure}{.25\textwidth}
  \includegraphics[width=\textwidth]{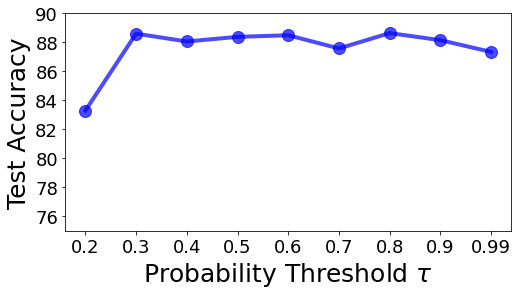}
\end{subfigure}
\caption{Analysis of $\lambda$ and $\tau$ on CIFAR-10 with 60\% symmetric label noise.}
\vspace{-10 pt}
\label{fig:abla}
\end{figure}

\section{Ablation Studies and Discussions} \label{sec:abla}
\subsection{Effects of hyperparameters}
The hyperparameter $\lambda$ controls the strength of the regularization to representations of data. 
A weak regularization is not able to address the memorization issue,
while a strong regularization makes the neural network mainly focus on optimizing the regularization term and ignoring optimizing the linear classifier.
Figure \ref{fig:abla} (left) shows the test accuracy with different $\lambda$.
The results are in line with the expectation that too strong or too weak regularizations leads to poor performance.

The $\tau$ is the confidence threshold for choosing two examples from the same classes.
When the score for the two examples exceeds the threshold, the two examples are considered as the correct pair.
Many wrong pairs are selected if $\tau$ is set too low.
Figure \ref{fig:abla} (right) shows the test accuracy with different $\tau$.
When we are always confident about any pairs ($\tau$=0), the model performance is reduced significantly ($\sim 20\%$).

\subsection{Effects of regularization functions}

To study the effect of the proposed regularization function,  
we compare the performance of Eq.~(\ref{loss:noisyctr}) to Eq.~(\ref{eq:ori_noisectr}).
Empirical results are consistent with the previous gradient analysis and they are shown in Table \ref{abla:ctrloss}.
Our proposed regularization function Eq.~(\ref{loss:noisyctr}) outperforms Eq.~(\ref{eq:ori_noisectr}) by a large margin across all noise levels.
Thus, learning data representations with Eq.~(\ref{loss:noisyctr}) can avoid wrong pairs dominating the representation learning.

\subsection{Other contrasitve learning frameworks}
Since the InfoMax principle \cite{oord2019representation} of contrastive learning and the gradient analysis can apply to other contrastive learning frameworks, we apply CTRR to other contrastive learning frameworks.
Table \ref{tab:ssl} shows that our principle is not limited to the SimSiam framework but can also be applied on other contrastive learning frameworks.
Since BYOL leverages an additional exponential moving average model to learn representations, the performance of CTRR with BYOL performs better, compared with SimSiam.
CTRR works slightly worse under SimCLR than the other two frameworks.
For its implementation, we simply replace the inner product of positive representations in SimCLR with our regularization function Eq.~(\ref{loss:noisyctr}) and keep the SimCLR objective function from negative pairs the same.
A study on how negative pairs from SimCLR affects representation learning in presence of the label noise is beyond the scope of this paper.

\subsection{Combination with other methods}
Furthermore, CTRR is orthogonal to label correction techniques~\cite{DBLP:conf/iclr/ZhangZW0021, DBLP:conf/nips/LiuNRF20}. In other words, our method can be integrated with these techniques to further boost learning performances. 
Specifically, we use the basic label correction strategy following \cite{noiseagainstnoise} that labels are replaced by weighted averaged of both model predictions and original labels,
where weights are scaled sample losses.
In Table \ref{abla:labelcorrection}, we show that the performance is improved after enabling a simple label correction technique.

Note that GCE \cite{DBLP:conf/nips/ZhangS18} is a partial noise-robust loss function implicitly combined with CE and MAE.
It is of interest to re-validate the loss function GCE along with our proposed regularization function.
We show the performance of a combination of our method and GCE \cite{DBLP:conf/nips/ZhangS18} in Table \ref{abla:rep_gce}.
With representations induced by our proposed method,  there is a significant improvement on GCE, which demonstrates the effectiveness of the learned representations.
Meanwhile, the success of this combination implies that our proposed method is beneficial to other partial noise-robust loss functions.

\begin{table}[t]
    \centering
    \resizebox{0.5\textwidth}{!}{
    \begin{tabular}{c|cccc}
    \toprule
          \multicolumn{1}{c|}{Label Correction} & \multicolumn{4}{c}{CIFAR-10}  \\ 
         Technique & 20\% & 40\% & 60\%  & 80\% \\
         \hline\hline
        \xmark & $93.05_{\pm 0.32}$ &$92.16_{\pm 0.31}$ &$87.34_{\pm 0.84}$ &$83.66_{\pm 0.52}$ \\
        \cmark & $\mathbf{93.32_{\pm 0.11}}$ &$ \mathbf{92.76_{\pm 0.67}}$ &$\mathbf{89.23_{\pm 0.18}}$ &$\mathbf{85.40_{\pm 0.93}}$ \\
        \bottomrule
    \end{tabular}
    }
    \caption{\cmark/\xmark\ indicates the label correction technique is enabled/disabled.}
    \label{abla:labelcorrection}
\end{table}

\begin{table}[t]
    \centering
    \resizebox{0.5\textwidth}{!}{
    \begin{tabular}{c|cccc}
    \toprule
          \multirow{2}{*}{Method} & \multicolumn{4}{c}{CIFAR-10}  \\ 
          & 20\% & 40\% & 60\%  & 80\% \\
         \hline\hline
         GCE & $91.22_{\pm 0.25}$ &$89.26_{\pm 0.34}$ &$85.76_{\pm 0.58}$ &$70.57_{\pm 0.83}$ \\
        CTRR & $93.05_{\pm 0.32}$ &$92.16_{\pm 0.31}$ &$87.34_{\pm 0.84}$ &$83.66_{\pm 0.52}$ \\
        CTRR+GCE & $\mathbf{93.94_{\pm 0.09}}$ &$ \mathbf{93.06_{\pm 0.29}}$ &$\mathbf{92.79_{\pm 0.06}}$ &$\mathbf{90.25_{\pm 0.40}}$ \\
        \bottomrule
    \end{tabular}
    }
    \caption{The performance of the model with respect to GCE, CTRR and CTRR+GCE.}
    \vspace{-12 pt}
    \label{abla:rep_gce}
\end{table}

\section{Related Work}
In this section, we briefly review existing approaches for learning with label noise.

Noise-robust loss functions are designed to achieve a small error on clean data instead of corrupted data while training on the noisy training data \cite{DBLP:journals/corr/abs-2007-08199,DBLP:conf/icml/MaH00E020,DBLP:conf/nips/AmidWAK19}.
Mean absolute error (MAE) is robust to label noise \cite{DBLP:conf/aaai/GhoshKS17} but  it is not able to solve complicated classification tasks.
The determinant based mutual information loss $L_\text{DMI}$ is proved to be robust to label noise \cite{ldmi} but it only works on the instance-independent label noise.
The generalized cross entropy (GCE) \cite{DBLP:conf/nips/ZhangS18} takes advantages of MAE and implicitly combined it with CE.
The symmetric cross entropy (SL) \cite{DBLP:conf/iccv/0001MCLY019} designs a noise-robust reverse cross entropy loss and explicitly combines it with CE.
However, they have not completely addressed the issue as CE is prone to memorizing corrupted labels.
The LIMIT \cite{DBLP:conf/icml/HarutyunyanRSG20} proposes to add noise to gradients to address the memorization issue.
SLN \cite{noiseagainstnoise} proposes to combat label noise by adding noise to labels of data.
However, they may suffer from underfitting problems.

There are many different contrasitve regularization functions and architectures proposed to learn representations such as SimCLR \cite{DBLP:conf/nips/ChenKSNH20}, MoCo \cite{mocov2}, BYOL \cite{byol}, SimSiam \cite{simsiam} and SupCon \cite{DBLP:conf/nips/KhoslaTWSTIMLK20}, where SupCon is to learn supervised representations with clean labels while others focus on learning self-supervised representations without labels.
We aim to learn representations with noisy labels.
We mainly follow the SimSiam framework, but our method is not limited to the SimSiam framework.
Recently, some methods existing methods \cite{Li_2021_ICCV, DBLP:journals/data/CiortanDP21, DBLP:conf/cvpr/0001L21, DBLP:conf/iclr/0001XH21} leverage contrasitve representation learning to address noisy label problems.
Compared to their methods, we theoretically analyze the benefits of learning such contrastive representations and we focus on addressing a fundamental issue of how to avoid wrong contrastive pairs dominating the representation learning.

There are many other methods for learning with noisy labels.
Sample selection methods such as Co-teaching \cite{DBLP:conf/nips/HanYYNXHTS18}, Co-teaching+ \cite{DBLP:conf/icml/Yu0YNTS19}, SELFIE \cite{DBLP:conf/icml/SongK019}, and JoCoR \cite{DBLP:conf/cvpr/WeiFC020} are selecting small loss examples to update models where they treat small loss examples as clean ones.
Loss correction methods such as Forward/Backward method \cite{DBLP:conf/cvpr/PatriniRMNQ17} modify the sample loss based on a noise transition matrix.
Some works propose to improve the estimation of the noise transition matrix such as T-Revision \cite{DBLP:conf/nips/XiaLW00NS19} and Dual T \cite{DBLP:conf/nips/YaoL0GD0S20}.
Label correction methods such as ELR \cite{DBLP:conf/nips/LiuNRF20}, M-DYR-H \cite{DBLP:conf/icml/ArazoOAOM19} and PENCIL \cite{DBLP:conf/cvpr/YiW19} replace noisy labels with pseudo-labels using different strategies.
Methods like DivideMix \cite{DBLP:conf/iclr/LiSH20} combine the sample selection, label correction and semi-supervised techniques and empirically demonstrate their success to combat noisy labels.

\section{Conclusion}
We present a simple but effective CTRR to address the memorization issue.
Our theoretical analysis indicates that CTRR induces noise-robust representations without suffering from the underfitting problem.
From algorithmic perspectives, we propose a novel regularization function to avoid adverse effects from wrong pairs.
The empirical results also demonstrate the effectiveness of CTRR.
On the one hand, we show the potential combinations of existing methods to improve the model performance.
On the other hand, we evaluate our method under different contrastive learning frameworks.
Both of them reveal the flexibility of our method and the importance of correctly regularizing data representations.
We believe that CTRR can be jointly used with other existing methods to better solve machine learning tasks where there exists label noise.

\section*{Acknowledgement}
Li Yi was supported by NSERC Discovery Grants Program.
Sheng Liu was partially supported by NSF grant DMS 2009752, NSF NRT-HDR Award 1922658 and Alzheimer’s Association grant AARG-NTF-21-848627.
Boyu Wang was supported by NSERC Discovery Grants Program.

{\small
\bibliographystyle{ieee_fullname}
\bibliography{main}
}

\def\confName[final]{CVPR}
\def\confYear{2022}

\appendix
\section*{Appendix}

\section{Experiment Details} \label{sec:experiment}

\subsection{Implementation Details}
For CIFAR datasets, we use the model PreAct Resnet18 \cite{DBLP:conf/cvpr/resnethe}.
For ANIMAL-10N, we use a random initialized model Resnet18 \cite{DBLP:conf/cvpr/resnethe}.
For Clothing1M, we use an ImageNet pre-trained model  Resnet18\cite{DBLP:conf/cvpr/resnethe}.
We illustrate our framework in Figure \ref{fig:framework}.
The projection MLP is 3-layer MLP and the prediction MLP is 2-layer MLP as proposed in Simsiam \cite{simsiam}.
We use weak augmentations $\mcal{A}_w:\mcal{X}\rightarrow \mcal{X}$ including random resized crop and random horizontal flip for optimizing the cross entropy loss $\mcal{L}_\text{ce}$. 
Following SimSiam \cite{simsiam} \cite{DBLP:conf/nips/ChenKSNH20}, we use a strong augmentation $\mcal{A}_s:\mcal{X}\rightarrow \mcal{X}$ applied on images twice for optimizing the contrastive regularization term $\mcal{\widetilde L}_\text{ctr}$.
Specifically, $\{z_i\} = f \big (\mcal{A}_s(\{x_i\})\big )$ and $\{q_i\}=h\big(f\big(\mcal{A}_s(\{x_i\} \big)\big)$ for every example $x_i$, where one strong augmented image is for calculating $z$ and another is for calculating $q$.

\subsection{Algorithm}
According to our gradient analysis on two different clean images $x_i,x_j$ with $y_i=y_j$ and a noisy image $x_m$ with $y_m=y_i$, apply the regularization function Eq.~(\ref{loss:noisyctr}) can avoid representation learning dominated by the wrong contrastive pair $(x_i,x_m)$.
The analysis does not cover the same image with two different augmentations.
When applying the strong augmentation twice, each image $x$ has two different augmentations $x^\prime,x^{\prime \prime}$.
The contrastive pair $(x^\prime,x^{\prime \prime})$ will also dominate the representation learning given the property of Eq.~(\ref{loss:noisyctr}).
However, focusing on learning similar representations of $(x^\prime,x^{\prime \prime})$ does not help to form a cluster structure in representation space.
As mentioned in \cite{wang2020understanding}, learning this self-supervised  representations causes representations of data distributed uniformly on the unit hypersphere.
Hence, we want the gradient from the pair $(x^\prime,x^{\prime \prime})$ to be smaller when their representations approach to each other.
We use the original contrastive regularization to regularize the pair $(x^\prime,x^{\prime \prime})$.
The pseudocode of the proposed method is given in Algorithm \ref{alg:code}.

\begin{figure}[t]
    \centering
    \includegraphics[width=0.5\textwidth, trim=0 0 0.5cm 0]{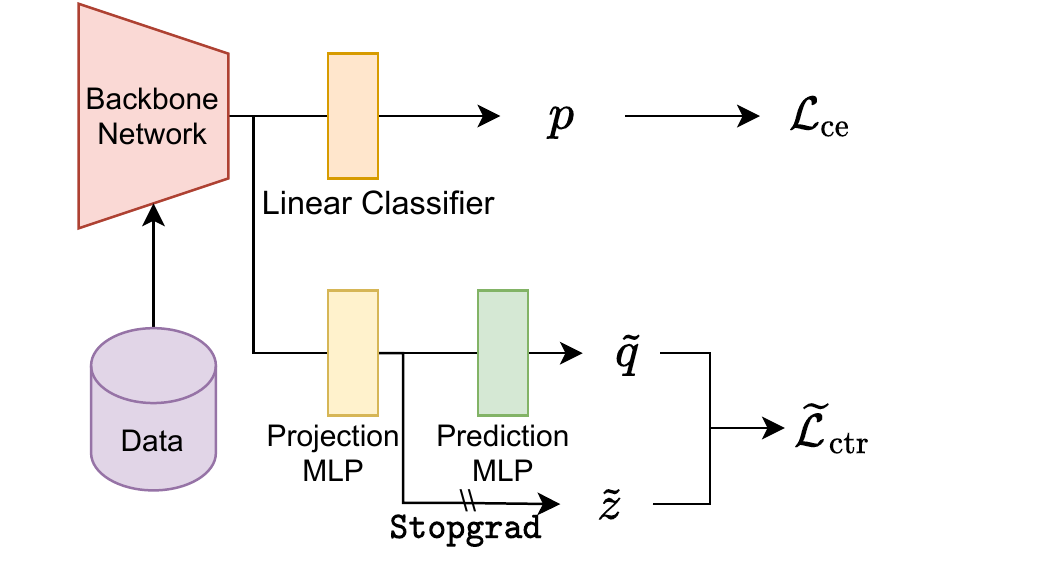}
    \caption{Illustration of our framework.}
    \label{fig:framework}
\end{figure}

\begin{algorithm*}[t]
\caption{CTRR Pseudocode in a PyTorch-like style}
\label{alg:code}
\definecolor{codeblue}{rgb}{0.25,0.5,0.5}
\definecolor{codekw}{rgb}{0.85, 0.18, 0.50}
\lstset{
  backgroundcolor=\color{white},
  basicstyle=\fontsize{7.5pt}{7.5pt}\ttfamily\selectfont,
  columns=fullflexible,
  breaklines=true,
  captionpos=b,
  commentstyle=\fontsize{7.5pt}{7.5pt}\color{codeblue},
  keywordstyle=\fontsize{7.5pt}{7.5pt}\color{codekw},
}
\begin{lstlisting}[language=python]
# Training
# f: backbone + projection mlp
# h: prediction mlp
# g: backbone + softmax linear classifier

for x, y in loader:  
    bsz = x.size(0)
    x1, x2 = strong_aug(x), strong_aug(x)  # strong random augmentation
    x3 = weak_aug(x) # weak random augmentation
    z1, z2 = f(x1), f(x2)  
    q1, q2 = h(z1), h(z2)
    p = g(x3)
    
    # avoid collapsing and gradient explosion
    z1 = torch.clamp(z1, 1e-4, 1-1e-4)
    z2 = torch.clamp(z2, 1e-4, 1-1e-4)
    p1 = torch.clamp(p1, 1e-4, 1-1e-4)
    p2 = torch.clamp(p2, 1e-4, 1-1e-4)
    
    # compute representations
    c1 = torch.matmul(q1, z2.t()) # B X B
    c2 = torch.matmul(q2, z1.t()) # B X B
    
    # compute contrastive loss for each pair
    m1 = torch.zeros(bsz, bsz).fill_diagonal_(1) # identity matrix
    m2 = torch.ones(bsz, bsz).fill_diagonal_(0) # 1-identity matrix
    #  - <i,i> + log(1-<i,j>)
    c1 = -c1*m1 + ((1-c1).log()) * m2 
    c2 = -c2*m1 + ((1-c2).log()) * m2
    c = torch.cat([c1, c2], dim=0)  # 2B X B
    
    # compute probability threshold
    probs_thred = torch.matmul(p, p.t()).fill_diagonal_(1).detach() # B X B
    mask = (probs_thred >= tau).float()
    probs_thred = probs_thred * mask
    # normalize the threshold
    weight = probs_thred / probs_thred.sum(1, keepdim=True)
    weight = weight.repeat((2, 1))  # 2B X B
    
    loss_ctr = (contrast_logits * weight).sum(dim=1).mean(0)
\end{lstlisting}
\end{algorithm*}

\subsection{Hyperparameters}
{\bf CIFAR.} Our method has two hyperparameters $\lambda$ and $\tau$. 
For each noise setting for CIFAR-10, we select the best hyperparameters: $\lambda$ from \{50, 130\} and $\tau$ from \{0.4, 0.8\}. 
For each noise setting for CIFAR-100, we select the best hyperparameters: $\lambda$ from \{50, 90\} and $\tau$ from \{0.05, 0.7\}. 
The batch size is set as $256$, and the learning rate is $0.02$ using SGD with a momentum of $0.9$ and a weight dacay of $0.0005$.

{\bf ANIMAL-10N \& Clothing1M.} 
For ANIMAL-10N, we set $\lambda=50$, $\tau=0.8$ and batch size is $256$.
The learning rate is set as $0.04$ with the same SGD optimizer as the CIFAR experiment.
For Clothing1M, we set $\lambda=90$, $\tau=0.4$ and batch size is $256$.
The learning rate is set as $0.06$ with the same SGD optimizer as above.

\section{Proofs of Theoretical Results}
\subsection{Proof for Theorem \ref{thm:relationship}}
\begin{theorem*}
Representations $Z$ learned by minimizing Eq.~(\ref{loss:ctr}) maximizes the mutual information $I(Z;X^+)$.
\end{theorem*}
\begin{proof}
We first decompose the mutual information $I(Z;X^+)$:
\begin{align*}
    I(Z;X^+) =&\expt_{Z,X^+}\log \frac{p(Z|X^+)}{p(Z)} \\
    =&\expt_{X^+}\expt_{Z|X^+}[\log{p(Z|X^+)}]  - \expt_{Z,X^+}[p(Z)]\\
    =&-\expt_{X^+}\big[H(Z|X^+)\big] + H(Z).
\end{align*}
The first term $\expt_{X^+}\big[H(Z|X^+)\big]$ measures the uncertainty of $Z|X^+$, which is minimized when $Z$ can be completely determined by $X^+$.
The second term $H(Z)$ measures the uncertainty of $Z$ itself and it is minimized when outcomes of $Z$ are equally likely.

We next show that $Z$ can be completely determined by $X^+$ when minimum of Eq.~(\ref{loss:ctr}) is achieved and uncertainty of $Z$ itself is maintained by an assumption about the framework.
By the Cauchy-Schwarz inequality, 
\begin{align*}
    \expt_{X,X^+}\big[\mcal{L}_\text{ctr}(X,X^+)\big] \geq  & \expt_{X,X^+}\big[\norm{\tilde q}_2\norm{\tilde z^+}_2 \\
    &+ \norm{\tilde  q^+}_2\norm{\tilde  z}_2 ] =-2.
\end{align*}
The equality is attained when $\tilde q=\tilde z^+$ and $\tilde q^+=\tilde z$ for all $x,x^+$ from the same class.
For any three images $x_1,x_2,x_3$ from the same class, we have:
\begin{equation*}
    f(x_1) = g(x_3),\quad f(x_2) = g(x_3),
\end{equation*}
where $g = h(f(\cdot))$.
We can find $f(x_1)=f(x_2)$ for any images $x_1,x_2$ from the same class.
The result can be easily extended to the general case: $f(X_1)=f(X_2)$ for any $(X_1,Y_1)\sim P(X,Y), (X_2,Y_2)\sim P(X,Y)$ with $Y_1=Y_2$.
Thus $Z$ can be determined by $X^+$ with the equation $Z=f(X^+)$, which minimizes $\expt_{X^+}\big[H(Z|X^+)\big]$.

When $p(Z=c_y|Y=y)=\frac{1}{|\mcal{Y}|}$, the entropy $H(Z)$ is maximized.
With extensive empirical results in Simsiam \cite{simsiam}, we assume the collapsed solutions are perfectly avoided by using the SimSiam framework.
By this assumption, $c_j \not= c_k$ for any $j\not=k$.
The model learns different clusters $c_y$ for different $y$ and representations with different labels have different clusters.
Therefore, for a balanced dataset, the outcomes of $Z$ are equally likely and it maximizes the second term $H(Z)$.
In summary, the learned representations by Eq.~(\ref{loss:ctr}) maximizes the mutual information $I(Z;X^+)$.
\end{proof}

\subsection{Proof for Theorem \ref{thm:1}}
\begin{theorem*}
Given a distribution $D(X,Y,\tilde Y)$ that is $(\epsilon, \gamma)$-Distribution, we have 
\begin{align}
    &I(X;Y) - \epsilon  \leq I(Z^\star;Y) \leq I(X;Y), \label{eq2}\\
    &I(Z^\star;\tilde Y)  \leq I(X;\tilde Y) - \gamma + \epsilon. \label{eq3}
\end{align}
\end{theorem*}
\begin{proof}
  The Theorem builds upon the Theorem 5 from \cite{tsai2021selfsupervised}.
We first provide the proof for the first inequality, which can also be obtained from \cite{tsai2021selfsupervised}.
Then we provide the proof for the second inequality.

For the first inequality, by adopting Data Processing Inequality in the Markov Chain $Y\leftrightarrow X\rightarrow Z$, we have $I(X;Y) \geq I(Z;Y)$ for any $Z\in \mcal{Z}$. Then, we have $I(X;Y) \geq I(Z^\star;Y)$.
Since $Z^\star=\argmax_{Z_\theta}I(Z_\theta;X^+)$, and $I(Z_\theta; X^+)$ is maximized at $I(X;X^+)$, then $I(Z^\star;X^+) = I(X; X^+)$ and $I(Z^\star;X^+|Y) = I(X; X^+|Y)$.
Meanwhile, use the result $I(Z^\star;X^+;Y) = I(X;X^+;Y)$, which is given by
\begin{align*}
    I(Z^\star;X^+;Y) & = I(Z^\star;X^+) - I(Z^\star;X^+|Y) \\
    &= I(X;X^+) - I(X;X^+|Y) \\
    &= I(X;X^+;Y),
\end{align*}
we have 
\begin{align} \label{eq:10}
    I(Z^\star;Y)=&I(X;X^+;Y) + I(Z^\star;Y|X^+)\nonumber \\ 
    =&I(X;Y)-I(X;Y|X^+)+I(Z^\star;Y|X^+).
\end{align}
Thus, by Eq.~(\ref{eq:10}) and the Definition \ref{assump:a1}, we get
\begin{equation}
    I(Z^\star;Y)\geq I(X;Y)-I(X;Y|X^+) \geq I(X;Y) -\epsilon
\end{equation}

Now we present the second inequality $I(Z^\star;\tilde Y)  \leq I(X;\tilde Y) - \gamma + \epsilon$.

Similarly, by Eq.~(\ref{eq:10}), we have
\begin{align}
    I(Z^\star;\tilde Y) &= I(X;\tilde Y)- I(X;\tilde Y|X^+)+I(Z^\star;\tilde Y|X^+)\\
    &\leq I(X;\tilde Y) - \gamma + I(Z^\star;\tilde Y|X^+) \\
    &\leq I(X;\tilde Y) - \gamma + I(Z^\star; Y|X^+) \\
    &\leq I(X;\tilde Y) - \gamma + \epsilon
\end{align},
where the first and the third inequalities are by the definition \ref{assump:a1};
the second inequality is by the Data Processing Inequality in the Markov Chain $\tilde Y\leftarrow Y \leftrightarrow X \rightarrow Z$.

\end{proof}

\subsection{Proof for Lemma \ref{lem:1}}
\begin{lemma*}
Consider a pair of random variables $(X,\tilde Y)$.
Let $\hat Y$ be outputs of any classifier based on inputs $Z_\theta$,
and $\tilde e=\mathbbm{1}\{\hat Y \not= \tilde Y \}$, where $\mathbbm{1}\{A\}$ be the indicator function of event $A$. 
Then, we have
\begin{equation*}
    \expt[\tilde e] \geq \frac{H(\tilde Y)-I(Z_\theta;\tilde Y) -H(\tilde e)}{\log(|\mcal{\tilde Y}|) - 1}.
\end{equation*}
\end{lemma*}
\begin{proof}
    If we are given any two of $\{\tilde e=1\}, \hat Y, \tilde Y$, the other one is known.
    By the properties of conditional entropy, $H(\tilde Y, \tilde e|\hat Y,Z_\theta)$ can be decomposed into the two equivalent forms.
    \begin{align}
        H(\tilde Y, \tilde e|\hat Y,Z_\theta)&=H(\tilde Y| \tilde e, \hat Y,Z_\theta) + H(\tilde e|\hat Y, Z_\theta)\nonumber \\
        &=\underbrace{H(\tilde Y| \tilde e, \hat Y,Z_\theta)}_\text{0} + H(\tilde Y|\hat Y,Z_\theta) \label{eq:17}
    \end{align}
    The first equality can also be decomposed into another form:
    \begin{align}
     &H(\tilde Y, \tilde e|\hat Y,Z_\theta)\nonumber \\
       =& H(\tilde Y| \tilde e, \hat Y,Z_\theta) + H(\tilde e|\hat Y, Z_\theta)\nonumber\\
       =&p(\tilde e =1)H(\tilde Y| \tilde e=1, \hat Y,Z_\theta)\nonumber \\
       &+ p(\tilde e =0)\underbrace{H(\tilde Y| \tilde e=0, \hat Y,Z_\theta)}_\text{0}+ H(\tilde e|\hat Y, Z_\theta)\nonumber \\
       =&p(\tilde e=1)H(\tilde Y| \tilde e=1, \hat Y,Z_\theta)+ H(\tilde e|\hat Y, Z_\theta) \label{eq:18}
    \end{align}
    Relating Eq.~(\ref{eq:17}) to Eq.~(\ref{eq:18}), we have
    \begin{align*}
        \expt[\tilde e]&=\frac{H(\tilde Y|\hat Y,Z_\theta)-H(\tilde e|\hat Y, Z_\theta)}{H(\tilde Y| \tilde e=1, \hat Y,Z_\theta)} \\
        & \geq \frac{H(\tilde Y|\hat Y,Z_\theta)-H(\tilde e|\hat Y, Z_\theta)}{\log{(|\mcal{Y}|-1)}} \\
        & \geq \frac{H(\tilde Y|\hat Y,Z_\theta)-H(\tilde e)}{\log{(|\mcal{Y}|-1)}} \\
        & =\frac{H(\tilde Y)-I(\tilde Y;Z_\theta, \hat Y)-H(\tilde e)}{\log{(|\mcal{Y}|-1)}} \\
        &=\frac{H(\tilde Y)-I(\tilde Y;Z_\theta)-H(\tilde e)}{\log{(|\mcal{Y}|-1)}}.
    \end{align*}
The first inequality is by $H(\tilde Y| \tilde e=1, \hat Y,Z_\theta) \leq \log{(|\mcal{Y}|-1)}$, where $\tilde Y$ can take at most $|\mcal{Y}|-1$ values.
For the second inequality,
\begin{align*}
    H(\tilde e|\hat Y, Z_\theta)&=H(\tilde e) - I(\tilde e; \hat Y, Z_\theta) \\
    & \leq  H(\tilde e).
\end{align*}
For the last equality,
\begin{align*}
    I(\tilde Y;Z_\theta, \hat Y) =& H(Z_\theta, \hat Y)-H(Z_\theta, \hat Y|\tilde Y)\\
    =&H(Z_\theta)+H(\hat Y|Z_\theta)\\
    &-H(Z_\theta|\tilde Y)-H(\hat Y|Z_\theta,\tilde Y) \\
    =&I(Z_\theta, \tilde Y) + I(\hat Y;\tilde Y|Z_\theta)\\
    =&I(Z_\theta, \tilde Y),
\end{align*}
where $I(\hat Y;\tilde Y|Z_\theta)=0$ given the Markov Chain $\tilde Y\leftarrow Y \leftrightarrow X \rightarrow Z \rightarrow \hat Y$:
\begin{align*}
    I(\hat Y;\tilde Y|Z_\theta)=&H(\hat Y|Z_\theta)-H(\hat Y|Z_\theta,\tilde Y)\\
    &=H(\hat Y|Z_\theta)-H(\hat Y|Z_\theta)=0.
\end{align*}
\end{proof}

\subsection{Proof for Lemma \ref{lem:2}}
\begin{lemma*} 
Let $R(X)=\inf_{g} \expt_{X,Y}[\mcal{L}(g(X),Y)]$ be the minimum risk over the joint distribution $X\times Y$, where $\mcal{L}(p,y)=\sum_{i=1}^{\mcal{Y}}y^{(i)}\log p^{(i)}$ is a CE loss  and $g$ is a function mapping from input space to label space.
Let $R(Z^\star)=\inf_{g^\prime} \expt_{Z^\star,Y}[\mcal{L}(g^\prime(Z^\star),Y)]$ be the minimum risk over the joint distribution $Z^\star\times Y$ and $g^\prime$ maps from representation space to label space. Then,
    \begin{equation*}
        R(Z^\star) \leq R(X) + \epsilon.
    \end{equation*}
\end{lemma*}
\begin{proof}
The lemma is given by the variational form of the conditional entropy $H(Y|Z^\star) = \inf_{g^\prime} \expt_{Z^\star,Y}[\mcal{L}(g^\prime(Z^\star),Y)]$ \cite{DBLP:conf/nips/FarniaT16, variationentropy}.
According to a property of mutual information,
\begin{equation*}
    I(A;B) = H(A) - H(A|B),
\end{equation*}
we have $R(Z^\star)= H(Y) - I(Z^\star;Y)$.
By the results of Theorem \ref{thm:1},
\begin{align*}
    R(Z^\star) \leq & H(Y)-I(X;Y) + \epsilon \\
    =& H(Y|X) = \inf_{g} \expt_{X,Y}[\mcal{L}(g(X),Y)].
\end{align*}
\end{proof}

\section{Gradients of Contrastive regularization Functions}
For the contrastive regularization function
\begin{equation*}
    \mcal{L}_\text{ctr}^\prime(x_i,x_j)=-\big(\frac{q_i}{\norm{q_i}_2}\cdot \frac{z_j}{\norm{z_j}_2} + \frac{q_j}{\norm{q_j}_2}\cdot \frac{z_i}{\norm{z_i}_2}\big),
\end{equation*}
we only consider the case $\mathbbm{1}\{p_i^\top p_j \geq \tau\}=1$ because $\mcal{L}^\prime_\text{ctr}(x_i,x_j)$ is not calculated in the algorithm when $\mathbbm{1}\{p_i^\top p_j \geq \tau\}=0$.
We assume that $h$ is an identity function and $x_i$, $x_j$ are from the same class for simplicity.

Let $a=\norm{q_i}_2$, $b=q_i$, $x=\frac{z_j}{\norm{z_j}_2}$ and $c=\frac{b}{a}$.
According to the equation $a^2=b^\top b$, we differentiate both side of the equation and get
\begin{equation} \label{eq:diff}
    2a\diff a = 2 b^\top \diff b.
\end{equation}
In the meanwhile, 
\begin{align*}
    \partial \big(\frac{b^\top x}{a} \big) &=\frac{\diff (b^\top x)a-\diff a b^\top x}{a^2} \\
    &\stackrel{(\ref{eq:diff})}{=} \frac{ax^\top \diff b}{a^2} - \frac{b^\top \diff b b^\top x}{a^3} \\
    &=\frac{x^\top \diff b}{a} - \frac{a^2c^\top x c^\top \diff b}{a^3}   \\
    &=\frac{1}{a}\big(x^\top - c^\top x c^\top \big) \diff b.
\end{align*}
Taking $a,b,c$ and $x$ back to the equation, we get the result
\begin{align*}
    \frac{\partial \mcal{L}_\text{ctr}^\prime (x_i,x_j)}{\partial q_i}&=-\frac{1}{\norm{q_i}_2}\big(\frac{q_j}{\norm{q_j}_2} -  (\frac{q_i^\top  q_j}{\norm{q_i}_2 \norm{q_j}_2}) \frac{q_i}{\norm{q_i}_2}   \big).
\end{align*}
Note that $z_i = \texttt{Stopgrad}(q_i)$ because of the identity map $h$.
Let $c_i=1/\norm{q_i}_2^2$ and then we have
\begin{equation*}
    \norm{ \frac{\partial \mcal{L}_\text{ctr}^\prime(x_i,x_j)}{\partial q_i}}_2^2
    = c_i(1-(\tilde q_i^\top \tilde q_j)^2).
\end{equation*}

Similarly, for the contrastive regularization function 
\begin{align*}
    \mcal{\widetilde L}_\text{ctr}(x_i,x_j)=&\bigg(\log{\big(1- \langle\,\frac{q_i}{\norm{q_i}_2}, \frac{z_j}{\norm{z_j}_2}  \rangle \big)} \\
    &+
     \log{\big(1- \langle\,\frac{q_j}{\norm{q_j}_2}, \frac{z_i}{\norm{z_i}_2}  \rangle \big)} \bigg),
\end{align*}
\begin{align*}
    \frac{\partial \mcal{\widetilde L}_\text{ctr}(x_i,x_j)}{\partial q_i} =&
    \frac{1}{1-\tilde q_i^\top \tilde q_j}\frac{\partial \mcal{L}_\text{ctr}^\prime (x_i,x_j)}{\partial q_i} \\
    =& c_i(1+\tilde q_i^\top \tilde q_j).
\end{align*}

\end{document}